\documentclass[12pt]{article}

\usepackage{amsmath}
\usepackage{amsthm}
\usepackage{amssymb}
\usepackage{mathtools}
\usepackage{bbm}
\usepackage{cancel}
\usepackage{enumerate}
\usepackage{hyperref}
\usepackage{natbib}

\usepackage{geometry}

\newtheorem{theorem}{Theorem}
\newtheorem{lemma}{Lemma}
\newtheorem{remark}{Remark}

\newtheorem{definition}{Definition}

\newcommand{\abs}[1]{\left\vert#1\right\vert}
\newcommand{\norm}[1]{\left\Vert#1\right\Vert}
\newcommand{\p}[1]{\left(#1\right)}

\newcommand{\R}{\mathbb{R}}

\newcommand{\N}{\mathbb{N}}

\newcommand{\Pb}{\mathbb{P}}
\DeclareMathOperator{\id}{id}

\DeclareMathOperator{\E}{\mathbb{E}}

\title{
  \vspace{-1cm}
  \Large{An Elementary Proof that Q-learning Converges Almost Surely}
}

\author{
  \normalsize Matthew T. Regehr
  \thanks{An early version of this work was submitted as author's CMPUT 653 course project in Winter 2021.} \\
  \normalsize University of Alberta \\
  \normalsize \texttt{mregehr@ualberta.ca}
  \and
  \normalsize Alex Ayoub \\
  \normalsize University of Alberta \\
  \normalsize \texttt{aayoub@ualberta.ca}
}

\date{
  \normalsize\today
  \vspace{-0.5cm}
}

\begin{document}
  \maketitle

  \section{Introduction}

  Watkins' and Dayan's Q-learning is a model-free reinforcement learning algorithm that iteratively refines an estimate for the optimal action-value function of an MDP by stochastically ``visiting'' many state-ation pairs \citep{watkins1992q}. Variants of the algorithm lie at the heart of numerous recent state-of-the-art achievements in reinforcement learning, including the superhuman Atari-playing deep Q-network \citep{mnih2015human}.

  The goal of this paper is to reproduce a precise and (nearly) self-contained proof that Q-learning converges. Much of the available literature leverages powerful theory to obtain highly generalizable results in this vein. However, this approach requires the reader to be familiar with and make many deep connections to different research areas. A student seeking to deepen their understand of Q-learning risks becoming caught in a vicious cycle of ``RL-learning Hell''. For this reason, we give a complete proof from start to finish using only one external result from the field of stochastic approximation, despite the fact that this minimal dependence on other results comes at the expense of some ``shininess''.

  \section{Related Works}

  The first proof that Q-learning converges with probability $1$ is outlined in \citep{watkins1989learning} and given more fully in \citep{watkins1992q}. The proof of \citep{tsitsiklis1994asynchronous} applies the theory of stochastic approximation to allow a far more general asynchronous structure. \citep{even2003learning} builds upon this work to derive more precise rates of convergence. Another approach by \citep{borkar2000ode} leverages the Lyapunov theory of ordinary differential equations to analyze a swath of stochastic approximation algorithms. Lastly, \citep{szepesvari1996generalized} analyzes Q-learning in the setting of generalized MDPs and focuses on the contractivity properties of dynamic programming operators.

  \section{Background}

  % In the interests of self-containment, this background section defines the basic objects in terms of which we describe and analyze the Q-learning algorithm.
  We make frequent use of standard measure theoretic and linear analytic notation and thus invite the reader to read Section \ref{sec:notation} upon encountering any unfamiliar symbols or terms.

  \subsection{Markov Decision Processes}

  \label{subsec:mdps}

  A typical formalization of environment in reinforcement learning---and the one we study here---is the Markov decision process (MDP). A reader familiar with the fundamentals of reinforcement learning may skip this subsection without issue.

  \begin{definition}
    A countable (finite) discounted MDP is a tuple $\langle\mathcal{S}, \mathcal{A}, P, r, \gamma\rangle$ where $\mathcal{S}$ and $\mathcal{A}$ are countable (finite) sets of ``states'' and ``actions'' respectively, $P : \mathcal{S} \times \mathcal{A} \to \Delta(\mathcal{S})$ is a ``transition kernel'', $r \in \ell^\infty(\mathcal{S} \times \mathcal{A})$ represents ``rewards'', and $\gamma \in [0, 1)$ is a ``discount rate''.
  \end{definition}

  In order to design agents that make ``good'' decisions when interacting with an MDP, we would like to somehow measure the value of making certain decisions in certain states. A convenient approach to measuring value relies on the fixed point theory of so-called ``dynamic programming'' operators. The following class of operators will serve our purposes nicely.

  \begin{definition}
    The ``Bellman optimality operator'' of an MDP $M = \langle\mathcal{S}, \mathcal{A}, P, r, \gamma\rangle$ is
    \begin{align*}
      T^*_M : \ell^\infty(\mathcal{S} \times \mathcal{A}) \to \ell^\infty(\mathcal{S} \times \mathcal{A}),
        q \mapsto (s, a) \mapsto r(s, a) + \gamma \sum_{s' \in \mathcal{S}} P(s' | s, a)\sup_{a' \in \mathcal{A}} q(s', a').
    \end{align*}
  \end{definition}

  Incidentally, exact or even approximate knowledge of the fixed point\footnote{A fixed point of a map $f : \mathcal{X} \to \mathcal{X}$ is a point $x^* \in \mathcal{X}$ such that $f(x^*) = x^*$.} of the Bellman optimality operator is sufficient to act optimally or near-optimally\footnote{See Lemma I at \url{https://rltheory.github.io/lecture-notes/planning-in-mdps/lec6/}.}. For now, however, it is enough that a unique fixed point exists. The proof is a routine application of the well-known Banach fixed point theorem and can be found in Section \ref{subsec:mdps_proofs}.

  \begin{theorem} \label{thm:existence_of_q*}
    For any MDP $M$, $T^*_M$ admits a unique fixed point $q^*_M$, which we refer to as the ``optimal action-value function'' for $M$.
  \end{theorem}

  The following bound will serve a useful purpose in proving our main theorem. As before, a proof can be found in Section \ref{subsec:mdps_proofs}.

  \begin{lemma}
    \label{lemma:q*_bound}

    For any MDP $M$ with rewards $r$ and discount rate $\gamma$,
    \begin{align*}
      \norm{q^*_M}_\infty \leq \frac{\norm{r}_\infty}{1 - \gamma}.
    \end{align*}
  \end{lemma}

  \subsection{Sampling Trajectories from an MDP}

  In order to compute Q-learning iterates, we would like to sample trajectories from a distribution that respects the dynamics of a given countable discounted MDP $M = \langle \mathcal{S}, \mathcal{A}, P, r, \gamma \rangle$. To that end, we require some statistical apparatus. Once again, the reader is referred to Section \ref{sec:notation} if any notation is unfamiliar.

  \begin{definition}
    The ``trajectory space'' of $M$ is the measurable space
    \begin{align*}
      (\Omega_M, \mathcal{F}_M) := \p{
        (\mathcal{S} \times \mathcal{A} \times \mathcal{S})^{\N_0},
        \bigotimes_{t \in \N_0} \mathcal{P}(\mathcal{S} \times \mathcal{A} \times \mathcal{S})
      }.
    \end{align*}
  \end{definition}
  \begin{definition}
    The ``trajectory process'' of $M$ is the sequence $(S_0, A_0, S'_0, S_1, A_1, S'_1, \dots)$ of $\mathcal{F}_M/\mathcal{P}(\mathcal{S})$ and $\mathcal{F}_M/\mathcal{P}(\mathcal{A})$-measurable projections defined by\footnote{For convenience, we suppress $M$ from the notation of the trajectory process as the correct meaning should always be deducible via ``type inference''.}
    \begin{align*}
      ((S_0, A_0, S'_0), (S_1, A_1, S'_1), \dots) := \id_{\Omega_M}.
    \end{align*}
  \end{definition}
  \begin{definition} \label{def:trajectory_measure}
    The set of ``trajectory measures'' on $M$, denoted $\Delta_T(M)$, is the set of probability measures $\Pb \in \Delta(\Omega_M, \mathcal{F}_M)$ satisfying
    \begin{align*}
      \Pb(S'_t = s'_t | S_0, A_0, S'_0, \dots, S_t, A_t) = P(s' | S_t, A_t)
    \end{align*}
    almost surely (a.s.) for any $s'_t \in \mathcal{S}$ and $t \in \N_0$.
  \end{definition}
  \begin{definition}
    The ``occurences'' of $(s, a) \in \mathcal{S} \times \mathcal{A}$ along a ``trajectory'' $\omega \in \Omega_M$ constitute
    \begin{align*}
      \mathcal{T}_{(s, a)}(\omega) := \{t \in \N_0 : (S_t, A_t)(\omega) = (s, a)\}.
    \end{align*}
  \end{definition}

  \section{The Q-learning Algorithm}

  Our overall goal is to design a reinforcement learning agent that makes good decisions in a given environment. To that end, we seek to develop an algorithm that closely approximates the optimal action-value function for a given MDP. Furthermore, we would like to do this without explicitly accessing an environment's transition kernel as these are frequently unavailable in real-world applications. On the other hand, many real-world environments permit the sampling of transitions and in fact we will use sampling to develop the Q-learning algorithm. In particular, by stochastically ``visiting'' many state-action pairs, we iteratively refine an estimate for $q^*_M$. The details of how we visit states and choose actions should not matter as long as our samples cover the state-action space sufficiently well. Altogether, these ideas form the basis of Watkins' and Dayan's Q-learning \citep{watkins1992q}.

  \begin{definition}[Q-learning]
    The ``Q-learning iterates'' on a finite MDP $M$ with
    % trajectory space $(\Omega_M, \mathcal{F}_M)$, trajectory process $(S_0, A_0, S'_0, \dots)$, and
    discount rate $\gamma$ induced by a ``stepsize'' sequence $\alpha = (\alpha_t)_{t \in \N_0}$ in $\R$ and a trajectory $\omega \in \Omega_M$ form the sequence\footnote{Similarly, we omit $M$ from the notation of the Q-learning iterates and rely instead upon context and prepositional phrases to make the underlying MDP unambiguous.} $(Q_t^\alpha(\omega))_{t \in \N_0}$ in $\ell^\infty(\mathcal{S} \times \mathcal{A})$ defined recursively by $Q_0^\alpha(\omega) \equiv \mathbf{0}$ and\footnote{We adopt the function ``currying'' convention $f(y; x) := f(x)(y)$ for $f : \mathcal{X} \to \mathcal{Y} \to \mathcal{Z}$, $x \in \mathcal{X}$, and $y \in \mathcal{Y}$.}
    \begin{align*}
      Q_{t + 1}^\alpha := (s, a; \omega) \mapsto
        \begin{cases}
          (1 - \alpha_t)Q_t^\alpha(s, a; \omega) + \alpha_t(r(s, a) + \gamma \max\limits_{a' \in \mathcal{A}}Q_t^\alpha(S'_t(\omega), a'; \omega))
            & \text{if } t \in \mathcal{T}_{(s, a)}(\omega) \\
          Q_t^\alpha(s, a; \omega)
            & \text{otherwise}
        \end{cases}
    \end{align*}
    for $t \in \N_0$.
  \end{definition}

  \begin{remark}
    While the construction of the Q-learning iterates depends explicitly on the states, actions, rewards, and discount rate of an MDP, it does not depend directly on the transition kernel of an MDP. This increases the flexibility of Q-learning and, as we will see later, does not preclude convergence as long as the trajectories are sampled from an appropriate distribution.
  \end{remark}

  \section{Convergence of Q-learning}

  Q-learning iterates in hand, we are ready to state the assumptions that lead to convergence.

  \begin{definition}
    \label{def:hypotheses}

    Let $M$ be an MDP. A trajectory measure $\Pb \in \Delta_T(M)$ (see Definition \ref{def:trajectory_measure}) and a sequence $(\alpha_t)_{t \in \N_0}$ in $[0, 1]$ are said to satisfy the Robbins--Monro condition when
    \begin{align*}
      \sum_{t \in \mathcal{T}_{(s, a)}(\omega)} \alpha_t = \infty
      \quad\text{and}\quad
      \sum_{t \in \mathcal{T}_{(s, a)}(\omega)} \alpha_t^2 < \infty.
    \end{align*}
    for all $(s, a) \in \mathcal{S} \times \mathcal{A}$ and $\Pb$-almost all $\omega \in \Omega_M$. The set of all such trajectory measure-stepsize sequence pairs is denoted $\nu(M)$.
  \end{definition}

  \begin{remark}
    \label{rmk:infinite_occurences}

    The condition that $\sum_{t \in \mathcal{T}_{(s, a)}(\omega)} \alpha_t = \infty$ requires that $\mathcal{T}_{(s, a)}(\omega)$ be infinite for all $(s, a) \in \mathcal{S} \times \mathcal{A}$ and $\Pb$-almost all $\omega \in \Omega_M$, i.e. the sampling strategy that produces the measure $\Pb$ must visit all state-action pairs infinitely often.
  \end{remark}

  At last, we have arrived at our main result. The proof is delayed until Subsection \ref{subsec:big_boi_proof} as only then will we be adequately equipped for the task.

  \begin{theorem} \label{thm:big_boi}
    Let $M$ be a finite MDP and let $(\Pb, \alpha) \in \nu(M)$ be a Robbins--Monro trajectory measure-stepsize sequence pair for $M$. Then the Q-learning iterates $(Q_t^\alpha(\omega))_{t \in \N_0}$ on $M$ converge uniformly to $q^*_M$ for $\Pb$-almost all $\omega \in \Omega_M$.
  \end{theorem}

  % Of course, the theorem is entirely vacuous if $\nu(M) = \varnothing$ and so the reader might demand to see an inhabitant of $\nu(M)$. In Section \ref{sec:non_vacuity}, we show that $\nu(M)$ is in fact richly populated.

  \subsection{The Action-Replay Processes}

  \label{subsec:arp}

  We begin our journey toward convergence by showing that an MDP $M$ can be recovered by a certain limiting process from a trajectory-dependent MDP whose whose optimal action-value functions track the Q-learning iterates on $M$. We will see that this construction serves as the primary proof device for proving the convergence of Q-learning.

  \begin{definition}
    \label{def:arp}

    The ``action-replay process'' of an MDP $M = \langle \mathcal{S}, \mathcal{A}, P, r, \gamma \rangle$ induced by a stepsize sequence $\alpha = (\alpha_t)_{t \in \N_0}$ and a trajectory $\omega \in \Omega_M$ is the MDP $\hat{M}^\alpha(\omega) := \langle \hat{\mathcal{S}}, \mathcal{A}, \hat{P}, \hat{r}, \gamma \rangle$ where $\hat{\mathcal{S}} := \mathcal{S} \times \N_0 \cup \{s_{\textrm{absorb}}\}$;
    \begin{align*}
      \hat{P}((S'_{t'}(\omega), t') | (s, t), a)
        & := \alpha_{t'}\prod_{\tau \in \mathcal{T}_{(s, a)}(\omega) \cap (t', t)}(1 - \alpha_\tau), \\
      \hat{P}(s_{\textrm{absorb}} | (s, t), a)
        & := \prod_{\tau \in \mathcal{T}_{(s, a)}(\omega) \cap [0, t)} (1 - \alpha_{\tau}), \text{and} \\
      \hat{P}(s_{\textrm{absorb}} | s_{\textrm{absorb}}, a)
        & := 1
    \end{align*}
    for $(s, a) \in \mathcal{S} \times \mathcal{A}$, $t \in \N_0$, and $t' \in \mathcal{T}_{(s, a)}(\omega) \cap [0, t)$ as well as $\hat{P}(\cdot | \cdot, \cdot) \equiv \mathbf{0}$ everwhere else; and, finally,
    \begin{align*}
      \hat{r}((s, t), a)
        := r(s, a)\sum_{t' \in \mathcal{T}_{(s, a)}(\omega) \cap [0, t)} \hat{P}((S'_{t'}(\omega), t') | (s, t), a)
    \end{align*}
    for $(s, t) \in \mathcal{S} \times \N_0$ and $a \in \mathcal{A}$ as well as $\hat{r}(\cdot, \cdot) \equiv \mathbf{0}$ everwhere else.
  \end{definition}

  Our next theorem reduces the analysis of Q-learning iterates to analysis of the optimal action-value function of an action-replay process.

  \begin{theorem}
    \label{thm:q*_arp}

    Let $M$ be a finite MDP, let $\alpha$ be a stepsize sequence, and let $(Q_t^\alpha(\omega))_{t \in \N_0}$ be the induced Q-learning iterates on $M$. For every $\omega \in \Omega_M$, $t \in \N_0$, and $(s, a) \in \mathcal{S} \times \mathcal{A}$,
    \begin{align*}
      q^*_{\hat{M}^\alpha(\omega)}((s, t), a) = Q_t^\alpha(s, a; \omega).
    \end{align*}
  \end{theorem}

  Before we prove the theorem, we strongly encourage the reader to prove the following lemma that shows that, while the dynamics of the action-replay processes may look intimidating at a first glance, their recursive form is much more pleasant to work with.

  \begin{lemma}
    \label{lemma:arp_dynamics}

    With all terms as in Definition \ref{def:arp}, $(s, a) \in \mathcal{S} \times \mathcal{A}$, and $\omega \in \Omega_M$, we have
    \begin{align*}
      \hat{P}((S'_{t'}(\omega), t') | (s, t + 1), a)
        % & = \alpha_{t'}\prod_{\tau \in \mathcal{T}_{(s, a)}(\omega) \cap (t', t + 1)} (1 - \alpha_\tau) \nonumber\\
        % & = \alpha_{t'}\prod_{\tau \in \mathcal{T}_{(s, a)}(\omega) \cap (t', t)} (1 - \alpha_\tau) \\
        % & = \hat{P}((S'_{t'}(\omega), t') | (s, t), a) \nonumber
        = \hat{P}((S'_{t'}(\omega), t') | (s, t), a)
    \end{align*}
    for any $t \notin \mathcal{T}_{(s, a)}(\omega)$ and $t' \in \mathcal{T}_{(s, a)}(\omega) \cap [0, t + 1)$ as well as
    \begin{align*}
      \hat{P}((S'_t(\omega), t) | (s, t + 1), a)
        % = \alpha_t\prod_{\tau \in \mathcal{T}_{(s, a)}(\omega) \cap \cancelto{\varnothing}{(t, t + 1)}}(1 - \alpha_\tau)
        = \alpha_t
    \end{align*}
    and
    \begin{align*}
      \hat{P}((S'_{t'}(\omega), t') | (s, t + 1), a)
        % & = \alpha_{t'}\prod_{\tau \in \mathcal{T}_{(s, a)}(\omega) \cap (t', t + 1)}(1 - \alpha_\tau) \nonumber\\
        % & = (1 - \alpha_t)\alpha_{t'}\prod_{\tau \in \mathcal{T}_{(s, a)}(\omega) \cap (t', t)}(1 - \alpha_\tau) \\
        % & = (1 - \alpha_t)\hat{P}((S'_{t'}(\omega), t') | (s, t), a) \nonumber
        = (1 - \alpha_t)\hat{P}((S'_{t'}(\omega), t') | (s, t), a)
    \end{align*}
    for any $t \in \mathcal{T}_{(s, a)}(\omega)$ and $t' \in \mathcal{T}_{(s, a)}(\omega) \cap [0, t)$.
  \end{lemma}

  \begin{proof}[Proof of Theorem \ref{thm:q*_arp}]
    Fix $\omega \in \Omega_M$ and let $\hat{M}^\alpha(\omega) = \langle \hat{\mathcal{S}}, \mathcal{A}, \hat{P}, \hat{r}, \gamma \rangle$.

    We begin by establishing an extremely useful form for the optimal action-values of $\hat{M}^\alpha(\omega)$. To that end, notice that, for any $a \in \mathcal{A}$,
    \begin{align*}
      q^*_{\hat{M}^\alpha(\omega)}(s_{\textrm{absorb}}, a)
        & = T^*_{\hat{M}^\alpha(\omega)}q^*_{\hat{M}^\alpha(\omega)}(s_{\textrm{absorb}}, a) \\
        & = \hat{r}(s_{\textrm{absorb}}, a) +
            \gamma
              \sum_{\sigma' \in \mathcal{S}_{\hat{M}^\alpha(\omega)}}\hat{P}(\sigma' | s_{\textrm{absorb}}, a)
              \max_{a' \in \mathcal{A}}
                q^*_{\hat{M}^\alpha(\omega)}(\sigma', a') \\
        & = \gamma\max_{a' \in \mathcal{A}} q^*_{\hat{M}^\alpha(\omega)}(s_{\textrm{absorb}}, a'),
    \end{align*}
    so, taking a maximum over $a \in \mathcal{A}$, we must have $\max_{a' \in \mathcal{A}} q^*_{\hat{M}^\alpha(\omega)}(s_{\textrm{absorb}}, a') = 0$ and hence
    \begin{align}
      \label{eq:arp_q*}
      q^*_{\hat{M}^\alpha(\omega)}((s, k), a)
        =&\ T^*_{\hat{M}^\alpha(\omega)}q^*_{\hat{M}^\alpha(\omega)}((s, k), a) \nonumber\\
        =&\ \hat{r}((s, k), a) +
            \gamma
            \sum_{\sigma' \in \mathcal{S}_{\hat{M}^\alpha(\omega)}}
              \hat{P}(\sigma' | (s, k), a)
              \max_{a' \in \mathcal{A}}
                q^*_{\hat{M}^\alpha(\omega)}(\sigma', a') \nonumber\\
        =&\ r(s, a)\sum_{t' \in \mathcal{T}_{(s, a)}(\omega) \cap [0, k)} \hat{P}((S'_{t'}(\omega), t') | (s, k), a) + \nonumber\\
         &\ \gamma
            \hat{P}(s_{\textrm{absorb}} | (s, k), a)
              % \underbrace{
              %   \max_{a' \in \mathcal{A}}
              %     q^*_{\hat{M}^\alpha(\omega)}(s_{\textrm{absorb}}, a')
              % }_{= 0}
              \cancelto{0}{
                \max_{a' \in \mathcal{A}}
                  q^*_{\hat{M}^\alpha(\omega)}(s_{\textrm{absorb}}, a')
              } + \\
         &\ \gamma
            \sum_{t' \in \mathcal{T}_{(s, a)}(\omega) \cap [0, k)}
              \hat{P}((S'_{t'}(\omega), t') | (s, k), a)
              \max_{a' \in \mathcal{A}}
                q^*_{\hat{M}^\alpha(\omega)}((S'_{t'}(\omega), t'), a') \nonumber\\
        =&\ \sum_{t' \in \mathcal{T}_{(s, a)}(\omega) \cap [0, k)}
              \hat{P}((S'_{t'}(\omega), t') | (s, k), a)
              \p{
                r(s, a) +
                \gamma\max_{a' \in \mathcal{A}} q^*_{\hat{M}^\alpha(\omega)}((S'_{t'}(\omega), t'), a')
              } \nonumber
    \end{align}
    for any $(s, a) \in \mathcal{S} \times \mathcal{A}$ and $k \in \N_0$.
    \clearpage

    With this in mind, we now prove the theorem by induction on $t$. Since $[0, 0) = \varnothing$, Equation (\ref{eq:arp_q*}) yields $q^*_{\hat{M}^\alpha(\omega)}((s, 0), a) = 0 = Q_0^\alpha(s, a; \omega)$ for any $(s, a) \in \mathcal{S} \times \mathcal{A}$ and hence the base case holds. As for the inductive step, let $t \in \N_0$, assume the claim holds for $t$, and let $(s, a) \in \mathcal{S} \times \mathcal{A}$. We consider two cases.

    If $t \notin \mathcal{T}_{(s, a)}(\omega)$, then, by Equation (\ref{eq:arp_q*}) and Lemma \ref{lemma:arp_dynamics}, we have
    \begin{align*}
      q^*_{\hat{M}^\alpha(\omega)}((s, t + 1), a)
        & = \sum_{t' \in \mathcal{T}_{(s, a)}(\omega) \cap [0, t + 1)}
              \hat{P}((S'_{t'}(\omega), t') | (s, t + 1), a)
              \p{
                r(s, a) +
                \gamma\max_{a' \in \mathcal{A}} q^*_{\hat{M}^\alpha(\omega)}((S'_{t'}(\omega), t'), a')
              } \\
        & = \sum_{t' \in \mathcal{T}_{(s, a)}(\omega) \cap [0, t)}
              \hat{P}((S'_{t'}(\omega), t') | (s, t), a)
              \p{
                r(s, a) +
                \gamma\max_{a' \in \mathcal{A}} q^*_{\hat{M}^\alpha(\omega)}((S'_{t'}(\omega), t'), a')
              } \\
        & = q^*_{\hat{M}^\alpha(\omega)}((s, t), a) \\
        & = Q_t^\alpha(s, a; \omega) \\
        & = Q_{t + 1}^\alpha(s, a; \omega).
    \end{align*}

    Likewise, if $t \in \mathcal{T}_{(s, a)}(\omega)$, then, by Equation (\ref{eq:arp_q*}) and Lemma \ref{lemma:arp_dynamics},
    \begin{align*}
      q^*_{\hat{M}^\alpha(\omega)}((s, t + 1), a)
        =&\ \sum_{t' \in \mathcal{T}_{(s, a)}(\omega) \cap [0, t + 1)}
              \hat{P}((S'_{t'}(\omega), t') | (s, t + 1), a)
              \p{
                r(s, a) +
                \gamma\max_{a' \in \mathcal{A}} q^*_{\hat{M}^\alpha(\omega)}((S'_{t'}(\omega), t'), a')
              } \\
        =&\ (1 - \alpha_t)
            \sum_{t' \in \mathcal{T}_{(s, a)}(\omega) \cap [0, t)}
                \hat{P}((S'_{t'}(\omega), t') | (s, t), a)
                \p{
                  r(s, a) +
                  \gamma\max_{a' \in \mathcal{A}} q^*_{\hat{M}^\alpha(\omega)}((S'_{t'}(\omega), t'), a')
                } \\
         &\ + \alpha_t
            \p{
              r(s, a) +
              \gamma\max_{a' \in \mathcal{A}} q^*_{\hat{M}^\alpha(\omega)}((S'_t(\omega), t), a')
            } \\
        =&\ (1 - \alpha_t)q^*_{\hat{M}^\alpha(\omega)}((s, t), a) +
            \alpha_t
            \p{
              r(s, a) +
              \gamma\max_{a' \in \mathcal{A}} q^*_{\hat{M}^\alpha(\omega)}((S'_t(\omega), t), a')
            } \\
        =&\ (1 - \alpha_t)Q_t^\alpha(s, a; \omega) +
            \alpha_t
            \p{
              r(s, a) +
              \gamma\max_{a' \in \mathcal{A}} Q_t^\alpha(S'_t(\omega), a'; \omega)
            } \\
        =&\ Q_{t + 1}^\alpha(s, a; \omega)
    \end{align*}
    and hence the inductive step holds as well.
  \end{proof}

  At the beginning of Subsection \ref{subsec:arp}, we promised that an MDP can be recovered from its action-replay process via a limiting procedure; we now make good on that promise.

  \begin{theorem}
    \label{thm:arp_limit}

    Let $M = \langle \mathcal{S}, \mathcal{A}, P, r, \gamma \rangle$ be an MDP and let $(\Pb, \alpha) \in \nu(M)$ (recall Definition \ref{def:hypotheses}). Then, for any $(s, a, s') \in \mathcal{S} \times \mathcal{A} \times \mathcal{S}$ and $\Pb$-almost all $\omega \in \Omega_M$,
    \begin{align*}
      \hat{r}((s, t), a; \omega)
        \xrightarrow{t \to \infty} r(s, a)
    \end{align*}
    and
    \begin{align*}
      \sum_{t' \in \mathcal{T}_{(s, a)}(\omega) \cap [0, t)} \hat{P}((s', t') | (s, t), a; \omega)
        \xrightarrow{t \to \infty} P(s' | s, a)
    \end{align*}
    where $\hat{M}^\alpha(\omega) = \langle \hat{\mathcal{S}}, \mathcal{A}, \hat{P}(\omega), \hat{r}(\omega), \gamma \rangle$.
  \end{theorem}

  The proof rests on a classic result from the theory of stochastic approximation.

  \begin{theorem}[The Robbins--Monro Theorem]
    \label{thm:robbins_monro}

    For any familes of random variables $(\beta_t)_{t \in \N_0}$, $(\xi_t)_{t \in \N_0}$, and $(X_t)_{t \in \N_0}$ such that $(\beta_t)_{t \in \N_0}$ is non-negative and satisfies $\sum_{t \in \mathcal{T}_{(s, a)}(\omega)} \beta_t = \infty$ as well as $\sum_{t \in \mathcal{T}_{(s, a)}(\omega)} \beta_t^2 < \infty$ a.s., $\E[\xi_t] = \Xi$ for all $t \in \N_0$, $(\xi_t)_{t \in \N_0}$ is bounded a.s., and
    \begin{align*}
      X_{t + 1} = (1 - \beta_t) X_t + \beta_t \xi_t
    \end{align*}
    for all $t \in \N_0$, we have that $X_t \to \Xi$ a.s.
  \end{theorem}

  A statement and proof of the theorem can be found under Theorem 2.3.1 in \citep{nla.cat-vn954258} and its original, weaker variant (quadratic mean convergence rather than almost sure convergence) is stated and proved in \citep{robbins1951stochastic}.

  \begin{proof}[Proof of Theorem \ref{thm:arp_limit}]
    Fix $(s, a, s') \in \mathcal{S} \times \mathcal{A} \times \mathcal{S}$ and discard a $\Pb$-null set from $\Omega_M$ so that $\sum_{t \in \mathcal{T}_{(s, a)}(\omega)} \alpha_t = \infty$ and $\sum_{t \in \mathcal{T}_{(s, a)}(\omega)} \alpha_t^2 < \infty$ for $\omega \in \Omega_M$. Furthermore, for any $k \in \N_0$ and $\omega \in \Omega_M$, let $T_k(\omega)$ be the $k$\textsuperscript{th} smallest element of $\mathcal{T}_{(s, a)}(\omega)$ (where $T_0(\omega) := \min{\mathcal{T}_{(s, a)}(\omega)}$), which is well-defined by Remark \ref{rmk:infinite_occurences}.

    We now show that the reward limit holds. To that end, for $t \in \N_0$ and $\omega \in \Omega_M$, define
    \begin{align*}
      X_t(\omega)
       := \hat{r}((s, t), a; \omega)
        = r(s, a)\sum_{t' \in \mathcal{T}_{(s, a)}(\omega) \cap [0, t)}
            \hat{P}((S'_{t'}(\omega), t') | (s, t), a; \omega).
    \end{align*}
    Then, for any $t \in \N_0$ and $\omega \in \Omega_M$, $t \notin \mathcal{T}_{(s, a)}(\omega)$ implies
    \begin{align*}
      X_{t + 1}(\omega)
        & = r(s, a)\sum_{t' \in \mathcal{T}_{(s, a)}(\omega) \cap [0, t + 1)}
              \hat{P}((S'_{t'}(\omega), t') | (s, t + 1), a; \omega) \\
        & = r(s, a)\sum_{t' \in \mathcal{T}_{(s, a)}(\omega) \cap [0, t)}
              \hat{P}((S'_{t'}(\omega), t') | (s, t), a; \omega) \\
        & = X_t(\omega)
    \end{align*}
    by Lemma \ref{lemma:arp_dynamics}, whereas $t \in \mathcal{T}_{(s, a)}(\omega)$ implies that
    \begin{align*}
      X_{t + 1}(\omega)
        & = r(s, a)\sum_{t' \in \mathcal{T}_{(s, a)}(\omega) \cap [0, t + 1)}
              \hat{P}((S'_{t'}(\omega), t') | (s, t + 1), a; \omega) \\
        & = r(s, a)\p{(1 - \alpha_t)\sum_{t' \in \mathcal{T}_{(s, a)}(\omega) \cap [0, t)}
              \hat{P}((S'_{t'}(\omega), t') | (s, t), a; \omega) + \alpha_t} \\
        & = (1 - \alpha_t)X_t(\omega) + \alpha_t r(s, a)
    \end{align*}
    by Lemma \ref{lemma:arp_dynamics}. In particular, we have
    \begin{align*}
      X_{T_{k + 1}} = (1 - \alpha_{T_k})X_{T_k} + \alpha_{T_k}r(s, a)
    \end{align*}
    for all $k \in \N_0$. By Theorem \ref{thm:robbins_monro}, $X_{T_k}(\omega) \xrightarrow{k \to \infty} r(s, a)$ for $\Pb$-almost all $\omega \in \Omega_M$. Finally, since $(X_t)_{t \in \N_0}$ is constant between the terms of the subsequence $(X_{T_k})_{k \in \N_0}$, we have
    \begin{align*}
      \hat{r}((s, t), a; \omega) = X_t(\omega) \xrightarrow{t \to \infty} r(s, a)
    \end{align*}
    for $\Pb$-almost all $\omega \in \Omega_M$ as well.
    \clearpage

    Next, we show that the dynamics limit holds in an analogous fashion. To that end, for $t \in \N_0$ and $\omega \in \Omega_M$, define
    \begin{align*}
      Y_t(\omega)
        := \sum_{t' \in \mathcal{T}_{(s, a)}(\omega) \cap [0, t)} \hat{P}((s', t') | (s, t), a; \omega).
    \end{align*}
    Then, for any $t \in \N_0$ and $\omega \in \Omega_M$, $t \notin \mathcal{T}_{(s, a)}(\omega)$ implies $Y_{t + 1}(\omega) = Y_t(\omega)$ by Lemma \ref{lemma:arp_dynamics}, whereas $t \in \mathcal{T}_{(s, a)}(\omega)$ implies that
    \begin{align*}
      Y_{t + 1}(\omega)
        & = \sum_{t' \in \mathcal{T}_{(s, a)}(\omega) \cap [0, t + 1)}
              \hat{P}((s', t') | (s, t + 1), a) \\
        & = \sum_{t' \in \mathcal{T}_{(s, a)}(\omega) \cap [0, t + 1)}
              \mathbbm{1}(S'_{t'}(\omega) = s')
              \hat{P}((S'_{t'}(\omega), t') | (s, t + 1), a) \\
        & = (1 - \alpha_t)\sum_{t' \in \mathcal{T}_{(s, a)}(\omega) \cap [0, t)}
              \mathbbm{1}(S'_{t'}(\omega) = s')
              \hat{P}((S'_{t'}(\omega), t') | (s, t), a) +
            \alpha_t \mathbbm{1}(S'_t(\omega) = s') \\
        & = (1 - \alpha_t)\sum_{t' \in \mathcal{T}_{(s, a)}(\omega) \cap [0, t)}
              \hat{P}((s', t') | (s, t), a) +
            \alpha_t \mathbbm{1}(S'_t(\omega) = s') \\
        & = (1 - \alpha_t)Y_t(\omega) + \alpha_t \mathbbm{1}(S'_t(\omega) = s')
    \end{align*}
    by Lemma \ref{lemma:arp_dynamics}. In particular, we have
    \begin{align*}
      Y_{T_{k + 1}} = (1 - \alpha_{T_k})Y_{T_k} + \alpha_{T_k}\mathbbm{1}(S'_{T_k} = s')
    \end{align*}
    for all $k \in \N_0$. But, for any $k \in \N_0$,
    \begin{align*}
      \E[\mathbbm{1}(S'_{T_k} = s')]
        & = \Pb(S'_{T_k} = s') \\
        & = \sum_{t = 0}^\infty \Pb(T_k = t, S'_t = s') \\
        & = \sum_{t = 0}^\infty
              \Pb(\abs{\mathcal{T}_{(s, a)} \cap [0, t)} = k - 1, S_t = s, A_t = a, S'_t = s') \\
        & = \sum_{t = 0}^\infty
              \Pb(\abs{\mathcal{T}_{(s, a)} \cap [0, t)} = k - 1, S_t = s, A_t = a)P(s' | s, a) \\
        & = P(s' | s, a)\sum_{t = 0}^\infty \Pb(T_k = t) \\
        & = P(s' | s, a)
    \end{align*}
    since $\abs{\mathcal{T}_{(s, a)} \cap [0, t)}$ is a $\sigma(S_0, A_0, S'_0, \dots, S_{t - 1}, A_{t - 1})$-measurable random variable and since $\Pb$ is a trajectory measure on $M$. By Theorem \ref{thm:robbins_monro}, $Y_{T_k}(\omega) \xrightarrow{k \to \infty} P(s' | s, a)$ for $\Pb$-almost all $\omega \in \Omega_M$. As $(Y_t)_{t \in \N_0}$ is constant between the terms of the subsequence $(Y_{T_k})_{k \in \N_0}$,
    \begin{align*}
      \sum_{t' \in \mathcal{T}_{(s, a)}(\omega) \cap [0, t)} \hat{P}((s', t') | (s, t), a; \omega)
        = Y_t(\omega)
        \xrightarrow{t \to \infty} P(s' | s, a)
    \end{align*}
    for $\Pb$-almost all $\omega \in \Omega_M$ as well.
  \end{proof}

  \subsection{Proof of Theorem \ref{thm:big_boi}}

  \label{subsec:big_boi_proof}

  Having tamed the action-replay processes, all of the conceptual pieces are now in place to prove the convergence of Q-learning.
  % Since the optimal action-values for the ARP coincide with the Q-learning iterates, it just remains to show that, as the dynamics and rewards for the ARP tend to the dynamics and rewards for $M$, then so too tend the optimal action-values for the ARP toward those of $M$.
  For the sake of digestibility, we have factored out some of the technical heavy lifting into the following two lemmas.

  \begin{lemma}
    \label{lemma:low_level_bound}

    Let $M = \langle\mathcal{S}, \mathcal{A}, P, r, \gamma\rangle$ be an MDP, $\alpha = (\alpha_t)_{t \in \N_0}$ a stepsize sequence in $[0, 1]$, $\omega \in \Omega_M$, and $(s, a) \in \mathcal{S} \times \mathcal{A}$. For any $\tilde{t}, t \in \N_0$ with $\tilde{t} \leq t$,
    \begin{align*}
      \sum_{t' \in \mathcal{T}_{(s, a)}(\omega) \cap [0, \tilde{t})} \hat{P}((S'_{t'}(\omega), t') | (s, t), a)
        \leq e^{-\sum_{\tau \in \mathcal{T}_{(s, a)}(\omega) \cap [\tilde{t}, t)} \alpha_\tau}
    \end{align*}
    where $\hat{M}^\alpha(\omega) = \langle \hat{\mathcal{S}}, \mathcal{A}, \hat{P}, \hat{r}, \gamma \rangle$.
  \end{lemma}
  \begin{proof}
    Since $1 - \alpha \leq e^{-\alpha}$ for all $\alpha \in \R$,
    \begin{align*}
      \sum_{t' \in \mathcal{T}_{(s, a)}(\omega) \cap [0, \tilde{t})} \hat{P}((S'_{t'}(\omega), t') | (s, t), a)
        & =
            \sum_{t' \in \mathcal{T}_{(s, a)}(\omega) \cap [0, \tilde{t})}
              \alpha_{t'}\prod_{\tau \in \mathcal{T}_{(s, a)}(\omega) \cap (t', t)}(1 - \alpha_\tau) \\
        & \leq
            \prod_{\tau \in \mathcal{T}_{(s, a)}(\omega) \cap [0, t)}(1 - \alpha_\tau) +
            \sum_{t' \in \mathcal{T}_{(s, a)}(\omega) \cap [0, \tilde{t})}
              \alpha_{t'}\prod_{\tau \in \mathcal{T}_{(s, a)}(\omega) \cap (t', t)}(1 - \alpha_\tau) \\
        & = \prod_{\tau \in \mathcal{T}_{(s, a)}(\omega) \cap [\tilde{t}, t)}(1 - \alpha_\tau) \\
        & \leq \prod_{\tau \in \mathcal{T}_{(s, a)}(\omega) \cap [\tilde{t}, t)}e^{-\alpha_\tau} \\
        & = e^{-\sum_{\tau \in \mathcal{T}_{(s, a)}(\omega) \cap [\tilde{t}, t)}\alpha_\tau}
    \end{align*}
    where the second equality follows by induction on $\tilde{t}$ (we encourage the reader to check).
  \end{proof}

  \begin{lemma}
    \label{lemma:one_step_error}

    Let $M = \langle\mathcal{S}, \mathcal{A}, P, r, \gamma\rangle$ be a finite MDP, let $\alpha = (\alpha_t)_{t \in \N_0}$ be a stepsize sequence, let $\omega \in \Omega_M$, let $(Q_t := Q_t^\alpha(\omega))_{t \in \N_0}$ be the induced Q-learning iterates on $M$, and let $\tilde{t}, t \in \N_0$ with $\tilde{t} \leq t$. Then, for any $(s, a) \in \mathcal{S} \times \mathcal{A}$, $\abs{Q_t(s, a) - q^*_M(s, a)}$ is at most
    \begin{align*}
      \gamma \max_{t' \in [\tilde{t}, t)} \norm{Q_{t'} - q^*_M}_\infty +
      \norm{\hat{r}_t - r}_\infty +
      \p{\frac{\gamma \norm{r}_\infty}{1 - \gamma}}\p{
        \abs{\mathcal{S}}\norm{\hat{P}_t - P}_\infty +
        2e^{-\sum_{\tau \in \mathcal{T}_{(s, a)}(\omega) \cap [\tilde{t}, t)}\alpha_\tau}
      }
    \end{align*}
    where $\hat{M} := \hat{M}^\alpha(\omega) = \langle \hat{\mathcal{S}}, \mathcal{A}, \hat{P}, \hat{r}, \gamma \rangle$, $\hat{P}_t(s' | s, a) := \sum_{t' \in \mathcal{T}_{(s, a)}(\omega) \cap [0, t)} \hat{P}((s', t') | (s, t), a)$, and $\hat{r}_t(s, a) := \hat{r}((s, t), a)$.
  \end{lemma}
  \begin{proof}
    Fix $(s, a) \in \mathcal{S} \times \mathcal{A}$. By Theorem \ref{thm:q*_arp} and the triangle inequality,
    \begin{align*}
      |Q_t(s, a) - q^*_M(s, a)|
        = & \abs{T^*_{\hat{M}}q^*_{\hat{M}}((s, t), a) - T^*_Mq^*_M(s, a)} \\
        \leq
          & \abs{\hat{r}((s, t), a) - r(s, a)} + \\
          & \gamma\Bigg|
              \sum_{t' \in \mathcal{T}_{(s, a)}(\omega) \cap [0, t)}
                \hat{P}((S'_{t'}(\omega), t') | (s, t), a)
                \max_{a' \in \mathcal{A}} q^*_{\hat{M}}((S'_{t'}(\omega), t'), a') \\
          & \qquad\qquad -
              \sum_{s' \in \mathcal{S}}
                P(s' | s, a)
                \max_{a' \in \mathcal{A}} q^*_M(s', a')\Bigg|.
    \end{align*}
    But $\abs{\hat{r}((s, t), a) - r(s, a)} = \abs{\hat{r}_t(s, a) - r(s, a)} \leq \norm{\hat{r}_t - r}_\infty$ and, applying the triangle inequality once more,
    \begin{align*}
      \abs{
        \sum_{t' \in \mathcal{T}_{(s, a)}(\omega) \cap [0, t)}
          \hat{P}((S'_{t'}(\omega), t') | (s, t), a)
          \max_{a' \in \mathcal{A}} q^*_{\hat{M}}((S'_{t'}(\omega), t'), a')
        -
        \sum_{s' \in \mathcal{S}}
          P(s' | s, a)
          \max_{a' \in \mathcal{A}} q^*_M(s', a')
      }
    \end{align*}
    is bounded by the sum of (\ref{exp:term_a}) and (\ref{exp:term_b}) where
    \begin{align}
      \label{exp:term_a}
      \Bigg|
        \sum_{t' \in \mathcal{T}_{(s, a)}(\omega) \cap [0, t)}&
          \hat{P}((S'_{t'}(\omega), t') | (s, t), a)\p{
            \max_{a' \in \mathcal{A}} q^*_{\hat{M}}((S'_{t'}(\omega), t'), a') -
            \max_{a' \in \mathcal{A}} q^*_M(S'_{t'}(\omega), a')
          }
      \Bigg| \\
      \leq
      &
      \sum_{t' \in \mathcal{T}_{(s, a)}(\omega) \cap [\tilde{t}, t)}
        \hat{P}((S'_{t'}(\omega), t') | (s, t), a)
        \max_{a' \in \mathcal{A}}
          \abs{
            q^*_{\hat{M}}((S'_{t'}(\omega), t'), a') -
            q^*_M(S'_{t'}(\omega), a')
          } + \nonumber\\
      &
      \sum_{t' \in \mathcal{T}_{(s, a)}(\omega) \cap [0, \tilde{t})}
        \hat{P}((S'_{t'}(\omega), t') | (s, t), a)\p{
          \norm{q^*_{\hat{M}}}_\infty +
          \norm{q^*_M}_\infty
        } \nonumber\\
      \leq
      &
      \sum_{t' \in \mathcal{T}_{(s, a)}(\omega) \cap [\tilde{t}, t)}
        \hat{P}((S'_{t'}(\omega), t') | (s, t), a)
        \max_{a' \in \mathcal{A}}
          \abs{
            Q_{t'}(S'_{t'}(\omega), a') -
            q^*_M(S'_{t'}(\omega), a')
          } +
      \tag{Theorem \ref{thm:q*_arp}}\\
      &
      \p{\frac{\norm{\hat{r}}_\infty + \norm{r}_\infty}{1 - \gamma}}
      \sum_{t' \in \mathcal{T}_{(s, a)}(\omega) \cap [0, \tilde{t})}
        \hat{P}((S'_{t'}(\omega), t') | (s, t), a)
      \tag{Lemma \ref{lemma:q*_bound}}\\
      \leq
      &
      \max_{t' \in [\tilde{t}, t)}
        \norm{Q_{t'} - q^*_M}_\infty +
      \p{\frac{2\norm{r}_\infty}{1 - \gamma}}
      e^{-\sum_{\tau \in \mathcal{T}_{(s, a)}(\omega) \cap [\tilde{t}, t)}\alpha_\tau}
      \tag{Lemma \ref{lemma:low_level_bound}}
    \end{align}
    and
    \begin{align}
      \label{exp:term_b}
      \Bigg|
        \sum_{t' \in \mathcal{T}_{(s, a)}(\omega) \cap [0, t)}&
          \hat{P}((S'_{t'}(\omega), t') | (s, t), a)
          \max_{a' \in \mathcal{A}} q^*_M(S'_{t'}(\omega), a')
        -
        \sum_{s' \in \mathcal{S}}
          P(s' | s, a)
          \max_{a' \in \mathcal{A}} q^*_M(s', a')
      \Bigg| \\
      =
      &
      \abs{
        \sum_{s' \in \mathcal{S}}
          \p{
            \sum_{t' \in \mathcal{T}_{(s, a)}(\omega) \cap [0, t)}
              \hat{P}((s', t') | (s, t), a)
            -
            P(s' | s, a)
          }
          \max_{a' \in \mathcal{A}} q^*_M(s', a')
      } \nonumber\\
      \leq
      &
      \norm{q^*_M}_\infty
      \sum_{s' \in \mathcal{S}}
        \abs{\hat{P}_t(s' | s, a) - P(s' | s, a)}
      \nonumber\\
      \leq
      &
      \p{\frac{\abs{\mathcal{S}}\norm{r}_\infty}{1 - \gamma}}
      \norm{\hat{P}_t - P}_\infty
      \tag{Lemma \ref{lemma:q*_bound}}
    \end{align}
    (where the equality follows from the fact that $s' \neq S'_{t'}(\omega)$ implies $\hat{P}((s', t') | (s, t), a) = 0$), which yields the desired bound.
  \end{proof}

  It is time to finish the job. While most of the error terms provided by Lemma \ref{lemma:one_step_error} can be controlled in a straightforward manner via Theorem \ref{thm:arp_limit}, it is not immediately clear how to control $\max_{t' \in [\tilde{t}, t)} \norm{Q_{t'} - q^*_M}_\infty$. However, we will see that it may be subdued by repeatedly applying Lemma \ref{lemma:one_step_error} until a sufficiently small exponential coefficient is obtained.

  \begin{proof}[Proof Theorem \ref{thm:big_boi}]
    Taking finite unions of null sets as needed, discard a $\Pb$-null set from $\Omega_M$ so that, for all $(s, a) \in \mathcal{S} \times \mathcal{A}$ and $\omega \in \Omega_M$, $\sum_{t \in \mathcal{T}_{(s, a)}(\omega)} \alpha_t = \infty$ holds in addition to the conclusion of Theorem \ref{thm:arp_limit}. With this in mind, fix $\omega \in \Omega_M$, put $(Q_t)_{t \in \N_0} := (Q_t^\alpha(\omega))_{t \in \N_0}$, and let $(\hat{P}_t)_{t \in \N_0}$ as well as $(\hat{r}_t)_{t \in \N_0}$ be as in Lemma \ref{lemma:one_step_error}.

    Now, let $\epsilon > 0$ and choose $k \in \N$ sufficiently large so that
    \begin{align*}
      \gamma^{k + 1} \leq \frac{\epsilon(1 - \gamma)}{8 \norm{r}_\infty}
    \end{align*}
    (where $\cdot/0 := \infty$). Furthermore, by Theorem \ref{thm:arp_limit}, we may find $t_0 \in \N_0$ such that
    \begin{align*}
      \norm{\hat{r}_t - r}_\infty
        \leq \frac{\epsilon(1 - \gamma)}{4}
    \end{align*}
    and
    \begin{align*}
      \norm{\hat{P}_t - P}_\infty
        \leq \frac{\epsilon(1 - \gamma)^2}{4\gamma\abs{\mathcal{S}}\norm{r}_\infty}
    \end{align*}
    for $t \geq t_0$. Finally, as $\mathcal{S} \times \mathcal{A}$ is finite and as $\sum_{t \in \mathcal{T}_{(s, a)}(\omega)} \alpha_t = \infty$ for all $(s, a) \in \mathcal{S} \times \mathcal{A}$, we may choose $t_k \geq \dots \geq t_1 \geq t_0$ sufficiently far apart such that
    \begin{align*}
      e^{-\sum_{\tau \in \mathcal{T}_{(s, a)}(\omega) \cap [t_{i - 1}, t_i)} \alpha_\tau}
        \leq \frac{\epsilon(1 - \gamma)^2}{8\gamma \norm{r}_\infty}
    \end{align*}
    for all $i \in \{1, \dots, k\}$ and $(s, a) \in \mathcal{S} \times \mathcal{A}$.

    In particular, for any $i \in \{1, \dots, k\}$, $t \geq t_i$, and $(s, a) \in \mathcal{S} \times \mathcal{A}$,
    \begin{align*}
      \norm{\hat{r}_t - r}_\infty +
      &
      \p{\frac{\gamma\norm{r}_\infty}{1 - \gamma}}
      \p{
        \abs{\mathcal{S}}\norm{\hat{P}_{t} - P}_\infty +
        2e^{-\sum_{\tau \in \mathcal{T}_{(s, a)}(\omega) \cap [t_{i - 1}, t)} \alpha_\tau}
      } \\
      & \leq
        \norm{\hat{r}_t - r}_\infty +
        \p{\frac{\gamma\norm{r}_\infty}{1 - \gamma}}
        \p{
          \abs{\mathcal{S}}\norm{\hat{P}_{t} - P}_\infty +
          2e^{-\sum_{\tau \in \mathcal{T}_{(s, a)}(\omega) \cap [t_{i - 1}, t_i)} \alpha_\tau}
        } \\
      & \leq
        \frac{3}{4}\epsilon(1 - \gamma)
    \end{align*}
    since $(\alpha_t)_{t \in \N_0}$ is non-negative, so it follows by inductive application of Lemma \ref{lemma:one_step_error} that
    \begin{align*}
      \norm{Q_t - q^*_M}_\infty
        & \leq
            \gamma\max_{t' \in [t_k, t)}\norm{Q_{t'} - q^*_M}_\infty +
            \frac{3}{4}\epsilon(1 - \gamma) \\
        & \leq
            \gamma
            \max_{t' \in [t_k, t)}\p{
              \gamma\max_{t'' \in [t_{k - 1}, t')}\norm{Q_{t''} - q^*_M}_\infty +
              \frac{3}{4}\epsilon(1 - \gamma)
            } +
            \frac{3}{4}\epsilon(1 - \gamma) \\
        & =
            \gamma^2\max_{t' \in [t_{k - 1}, t)}\norm{Q_{t'} - q^*_M}_\infty +
            \frac{3}{4}\epsilon(1 - \gamma)(1 + \gamma) \\
        & \leq
            \dots \\
        & =
            \gamma^{k + 1}\max_{t' \in [t_0, t)}\norm{Q_{t'} - q^*_M}_\infty +
            \frac{3}{4}\epsilon(1 - \gamma)(1 + \gamma + \dots + \gamma^k) \\
        & \leq 
            \p{\frac{\epsilon(1 - \gamma)}{8 \norm{r}_\infty}}
            \p{\frac{2\norm{r}_\infty}{1 - \gamma}} +
            \frac{\frac{3}{4}\epsilon(1 - \gamma)}{1 - \gamma}
            \tag{Lemma \ref{lemma:q*_bound}} \\
        & = \epsilon
    \end{align*}
    for all $t \geq t_k$ and, with that, the beast has been slain.
  \end{proof}

  \clearpage

  \bibliographystyle{plainnat}
  \bibliography{qlc}

  \clearpage

  \appendix

  \section{Notation}

  \label{sec:notation}

  \subsection{Measure Theory}

  \begin{definition}
    A $\sigma$-algebra on a non-empty set $\mathcal{X}$ is collection of subsets $\mathcal{F}$ of $\mathcal{X}$ satisfying
    \begin{enumerate}[(i)]
      \item
        $\varnothing \in \mathcal{F}$;
      \item
        $\forall A \in \mathcal{F},\ \mathcal{X} \setminus A \in \mathcal{F}$; and
      \item
        $\forall A_1, A_2, \dots \in \mathcal{F},\ \bigcup_{n \in \N} A_n \in \mathcal{F}$.
    \end{enumerate}
    In this case, we call $(\mathcal{X}, \mathcal{F})$ a measurable space.
  \end{definition}

  \begin{definition}
    Given measurable spaces $(\mathcal{X}, \mathcal{F})$ and $(\mathcal{Y}, \mathcal{G})$ as well as a function $A : \mathcal{X} \to \mathcal{Y}$, the $\sigma$-algebra
    % \footnote{The reader is encouraged to verify that this is indeed a $\sigma$-algebra.}
    induced by $A$ is
    \begin{align*}
      \sigma_{\mathcal{G}}(A) := \{A^{-1}(G) : G \in \mathcal{G}\}
    \end{align*}
    and is usually denoted $\sigma(A)$ (when $\mathcal{G}$ is clear from context). Moreover,
    if $\sigma_{\mathcal{G}}(A) \subseteq \mathcal{F}$, we say that $A$ is $\mathcal{F}/\mathcal{G}$-measurable, $\mathcal{F}$-measurable, or just measurable for short.
  \end{definition}

  \begin{remark}
    Every non-empty set $\mathcal{X}$ admits at least one $\sigma$-algebra---namely $\mathcal{P}(\mathcal{X})$---and if $\{\mathcal{F}_i : i \in \mathcal{I}\}$ is a non-empty family of $\sigma$-algebras on $\mathcal{X}$, then $\bigcap_{i \in \mathcal{I}} \mathcal{F}_i$ is a $\sigma$-algebra
    on $\mathcal{X}$.
  \end{remark}

  \begin{definition}
    Given measurable spaces $(\mathcal{X}_i, \mathcal{F}_i)_{i \in \mathcal{I}}$, the product $\sigma$-algebra on $\bigtimes_{i \in \mathcal{I}} \mathcal{X}_i$
    \begin{align*}
      \bigotimes_{i \in \mathcal{I}} \mathcal{F}_i
        := \bigcap \{\mathcal{F} \text{ a $\sigma$-algebra on $\bigtimes_{i \in \mathcal{I}} \mathcal{X}_i$} : \forall i \in \mathcal{I},\ \text{$\pi_i$ is $\mathcal{F}/\mathcal{F}_i$-measurable} \}
    \end{align*}
    is the smallest $\sigma$-algebra with respect to which each projection $\pi_i : \bigtimes\limits_{j \in \mathcal{I}}\mathcal{X}_j \to \mathcal{X}_i$ is measurable.
  \end{definition}

  \begin{definition}
    A probability measure on a measurable space $(\mathcal{X}, \mathcal{F})$ is $\Pb : \mathcal{F} \to [0, \infty]$ s.t.
    \begin{enumerate}[(i)]
      \item
        $\Pb(\mathcal{X}) = 1$; and
      \item
        $\forall A_1, A_2, \dots \in \mathcal{F},\ (A_n)_{n \in \N} \text{ pairwise disjoint} \implies \Pb\p{\bigcup_{n \in \N} A_n} = \sum_{n \in \N} \Pb(A_n)$.
    \end{enumerate}
    Altogether, we call $(\mathcal{X}, \mathcal{F}, \Pb)$ a probability space and real-valued measurable functions on $\mathcal{X}$ are called random variables.
  \end{definition}

  \begin{definition}
    The set of probability measures on a measurable space $(\mathcal{X}, \mathcal{F})$ is denoted $\Delta(\mathcal{X}, \mathcal{F})$. If $\mathcal{X}$ is countable, we write $\Delta(\mathcal{X}) := \Delta(\mathcal{X}, \mathcal{P}(\mathcal{X}))$ for short.
  \end{definition}

  \subsection{Function Spaces}

  \begin{definition}
    Let $\mathcal{I}$ be a non-empty set. The supremum norm on $\mathcal{I} \to \R$ ($\R^{\mathcal{I}}$ for short) is
    \begin{align*}
      \norm{\cdot}_{\mathcal{I},\infty} : \R^\mathcal{I} \to [0, \infty],
      x \mapsto \sup_{i \in \mathcal{I}} \abs{x(i)}
    \end{align*}
    and the set of bounded real-valued functions on $\mathcal{I}$ is
    \begin{align*}
      \ell^\infty(\mathcal{I}) := \{x \in \R^\mathcal{I} : \norm{x}_{\mathcal{I},\infty} < \infty\}.
    \end{align*}
    Frequently, $\mathcal{I}$ is clear from context, in which case we write $\norm{\cdot}_\infty$ instead of $\norm{\cdot}_{\mathcal{I}, \infty}$.
  \end{definition}

  \section{Banach's Fixed Point Theorem}

  In order to prove Theorem \ref{thm:existence_of_q*}, we first need to know a little bit about metric spaces.

  \begin{definition}
    Let $E$ be non-empty. We call $d : E \times E \to [0, \infty)$ a metric on $E$ when
    \begin{enumerate}[(i)]
      \item
        $\forall x, y \in E,\ d(x, y) = 0 \iff x = y$;
      \item
        $\forall x, y \in E,\ d(x, y) = d(y, x)$; and
      \item
        $\forall x, y, z \in E,\ d(x, z) \leq d(x, y) + d(y, z)$.
    \end{enumerate}
    In this case, the pair $(E, d)$ is called a metric space.
  \end{definition}
  \begin{definition}
    A metric space $(E, d)$ is said to be complete when, for all sequences $(x_n)_{n \in \N_0}$ in $E$ satisfying
    \begin{align*}
      \forall \epsilon > 0,\ \exists n_0 \in \N_0,\ \forall n_1, n_2 \geq n_0,\ 
        d(x_{n_1}, x_{n_2}) \leq \epsilon
    \end{align*}
    (i.e. $(x_n)_{n \in \N_0}$ is a Cauchy sequence), we have that $d(x_n, x_\infty) \to 0$ for some $x_\infty \in E$.
  \end{definition}

  \begin{definition}
    Let $(E, d)$ be a metric space and let $\gamma \in [0, 1)$. We say that a map $T : E \to E$ is a $\gamma$-contraction on $(E, d)$ when
    \begin{align*}
      d(T(x), T(y)) \leq \gamma d(x, y)
    \end{align*}
    holds for all $x, y \in E$.
  \end{definition}

  \begin{theorem}[Banach's Fixed Point Theorem] \label{thm:banach}
    Let $(E, d)$ be a complete metric space and let $T : E \to E$ be a $\gamma$-contraction for some $\gamma \in [0, 1)$. Then $T$ admits a unique fixed point.
  \end{theorem}

  % \begin{proof}
  %   We begin with existence. Let $x_0 \in E$ be arbitrary. We claim that $(T^n x_0)_{n \in \N_0}$ is Cauchy. Indeed, let $\epsilon > 0$, choose $n_0 \in \N_0$ so large that
  %   \begin{align*}
  %     \frac{\gamma^{n_0}d(x_0, Tx_0)}{1 - \gamma} < \epsilon,
  %   \end{align*}
  %   and let $n_1, n_2 \geq n_0$. By symmetry of $d$, assume without loss of generality that $n_1 \leq n_2$. Then, applying $\gamma$-contractivity of $T$ inductively, we obtain
  %   \begin{align*}
  %     d(T^{n_1}x_0, T^{n_2}x_0)
  %       & \leq \gamma^{n_1}d(x_0, T^{n_2 - n_1}x_0) \\
  %       & \leq \gamma^{n_1}\sum_{k = 0}^{n_2 - n_1 - 1} d(T^k x_0, T^{k + 1} x_0) \\
  %       & \leq \gamma^{n_1}\sum_{k = 0}^{n_2 - n_1 - 1} \gamma^k d(x_0, Tx_0) \\
  %       & \leq \gamma^{n_0}\sum_{k = 0}^\infty \gamma^k d(x_0, Tx_0) \\
  %       & = \frac{\gamma^{n_0}d(x_0, Tx_0)}{1 - \gamma} \\
  %       & < \epsilon
  %   \end{align*}
  %   so that $(T^n x_0)_{n \in \N_0}$ is indeed Cauchy as claimed. By completeness of $(E, d)$, $d(T^n x_0, x_\infty) \to 0$ for some $x_\infty \in E$. In particular,
  %   \begin{align*}
  %     d(Tx_\infty, x_\infty)
  %       \leq TODO
  %   \end{align*}
  % \end{proof}

  The proof of Banach's fixed point theorem is a classic exercise in analysis. We omit it here but encourage the reader to try it on their own (hint: fix an arbitrary $x_0 \in E$ and show that $(T^n(x_0))_{n \in \N_0}$ is Cauchy by leveraging the fact that $\sum_{n = 0}^\infty \gamma^n$ is a convergent series).

  \section{Proofs of Results in Subsection \ref{subsec:mdps}}

  \label{subsec:mdps_proofs}

  In any case, the latter fixed point theorem is all we need to show that optimal action-value functions exist and are unique in an MDP.

  \begin{proof}[Proof of Theorem \ref{thm:existence_of_q*}]
    Let $M = \langle\mathcal{S}, \mathcal{A}, P, r, \gamma\rangle$.

    It is a straightforward exercise to verify that
    \begin{align*}
      d_\infty : \R^{\mathcal{S} \times \mathcal{A}} \times \R^{\mathcal{S} \times \mathcal{A}} \to [0, \infty],
      (q_1, q_1) \mapsto \norm{q_1 - q_2}_\infty
    \end{align*}
    is a metric on $\ell^\infty(\mathcal{S} \times \mathcal{A})$ and we omit the details.

    As for completeness, let $(q_n)_{n \in \N_0}$ be a Cauchy sequence in $(\ell^\infty(\mathcal{S} \times \mathcal{A}), d_\infty)$ and let $\epsilon > 0$. For each $(s, a) \in \mathcal{S} \times \mathcal{A}$ and $n_1, n_2 \in \N_0$, $\abs{q_{n_1}(s, a) - q_{n_2}(s, a)} \leq d_\infty(q_{n_1}, q_{n_2})$, which implies that $(q_n(s, a))_{n \in \N_0}$ is a Cauchy sequence in $\R$ and hence, by completeness of $\R$, converges to some $q_\infty(s, a) \in \R$; in particular, there is $n_{(s, a)} \in \N_0$ such that $\abs{q_n(s, a) - q_\infty(s, a)} \leq \frac{\epsilon}{2}$ for $n \geq n_{(s, a)}$. Furthermore, there is $n_0 \in \N_0$ for which $d_\infty(q_{n_1}, q_{n_2}) \leq \frac{\epsilon}{2}$ for $n_1, n_2 \geq n_0$. Hence
    \begin{align*}
      d_\infty(q_n, q_\infty)
        & = \sup_{(s, a) \in \mathcal{S} \times \mathcal{A}} \abs{q_n(s, a) - q_\infty(s, a)} \\
        & \leq
            \sup_{(s, a) \in \mathcal{S} \times \mathcal{A}} \p{
              d_\infty(q_n, q_{\max\{n_0, n_{(s, a)}\}}) +
              \abs{q_{\max\{n_0, n_{(s, a)}\}}(s, a) - q_\infty(s, a)}
            } \\
        & \leq
            \sup_{(s, a) \in \mathcal{S} \times \mathcal{A}} \p{
              \frac{\epsilon}{2} + \frac{\epsilon}{2}
            } \\
        & \leq \epsilon
    \end{align*}
    for $n \geq n_0$ and so $d_\infty(q_n, q_\infty) \to 0$. In particular, $d_\infty(q_{n_0}, q_\infty) < 1$ for some $n_0 \in \N_0$ and thus
    \begin{align*}
      \norm{q_\infty}_\infty
        \leq \norm{q_{n_0}}_\infty + d_\infty(q_{n_0}, q_\infty) < \infty,
    \end{align*}
    i.e. $q_\infty \in \ell^\infty(\mathcal{S} \times \mathcal{A})$ so that the latter is complete with respect to $d_\infty$ as claimed.

    Finally, we claim that $T^*_M$ is a $\gamma$-contraction on $(\ell^\infty(\mathcal{S} \times \mathcal{A}), d_\infty)$ as the conclusion will then follow immediately from Theorem \ref{thm:banach}. Indeed, for any $q_1, q_2 \in \ell^\infty(\mathcal{S} \times \mathcal{A})$ and $(s, a) \in \mathcal{S} \times \mathcal{A}$,
    \begin{align*}
      \abs{T^*_Mq_1(s, a) - T^*_Mq_2(s, a)}
        % & = \abs{
        %       r(s, a) + \gamma \sum_{s' \in \mathcal{S}}P(s' | s, a)\max_{a' \in \mathcal{A}} q_1(s', a') -
        %       (r(s, a) + \gamma \sum_{s' \in \mathcal{S}}P(s' | s, a)\max_{a' \in \mathcal{A}} q_2(s', a'))
        %     } \\
        & = \abs{
              \gamma
              \sum_{s' \in \mathcal{S}}
                P(s' | s, a)
                \p{\max_{a' \in \mathcal{A}} q_1(s', a') - \max_{a' \in \mathcal{A}} q_2(s', a')}
            } \\
        & \leq \gamma
            \sum_{s' \in \mathcal{S}}
              P(s' | s, a)
              \abs{\max_{a' \in \mathcal{A}} q_1(s', a') - \max_{a' \in \mathcal{A}} q_2(s', a')} \\
        & \leq \gamma
            \sum_{s' \in \mathcal{S}}
              P(s' | s, a)
              \max_{a' \in \mathcal{A}} \abs{q_1(s', a') - q_2(s', a')} \\
        & \leq \gamma
            \sum_{s' \in \mathcal{S}}
              P(s' | s, a)
              d_\infty(q_1, q_2) \\
        & = \gamma d_\infty(q_1, q_2),
    \end{align*}
    which implies that $d_\infty(T^*_M q_1, T^*_M q_2) \leq \gamma d_\infty(q_1, q_2)$ as desired.
  \end{proof}

  Lastly, the proof of Lemma \ref{lemma:q*_bound} follows from a straightforward calculation.

  \begin{proof}[Proof of Lemma \ref{lemma:q*_bound}]
    Let $M = \langle\mathcal{S}, \mathcal{A}, P, r, \gamma\rangle$. Then, for any $(s, a) \in \mathcal{S} \times \mathcal{A}$,
    \begin{align*}
      \abs{q^*_M(s, a)}
        & = \abs{T^*_Mq^*_M(s, a)} \\
        & = \abs{r(s, a) + \gamma\sum_{s' \in \mathcal{S}} P(s' | s, a)\sup_{a' \in \mathcal{A}} q^*_M(s', a')} \\
        & \leq \abs{r(s, a)} + \gamma\sum_{s' \in \mathcal{S}} P(s' | s, a)\abs{\sup_{a' \in \mathcal{A}} q^*_M(s', a')} \\
        & \leq \norm{r}_\infty + \gamma\norm{q^*_M}_\infty\sum_{s' \in \mathcal{S}} P(s' | s, a) \\
        & = \norm{r}_\infty + \gamma\norm{q^*_M}_\infty.
    \end{align*}
    In particular, $\norm{q^*_M}_\infty \leq \norm{r}_\infty + \gamma\norm{q^*_M}_\infty$ and hence $\norm{q^*_M}_\infty \leq \frac{\norm{r}_\infty}{1 - \gamma}$.
  \end{proof}

  % \section{Non-Vacuity of Theorem \ref{thm:big_boi}}

  % \label{sec:non_vacuity}

  % As promised before, we will show that $\nu(M)$ is non-empty for an arbitrary MDP $M$.

  % \begin{definition}
  %   The set of policies for an MDP $M = \langle\mathcal{S}, \mathcal{A}, P, r, \gamma\rangle$ is
  %   \begin{align*}
  %     \Pi_M := (\mathcal{S} \times \mathcal{A})^* \times \mathcal{S} \to \Delta(\mathcal{A})
  %   \end{align*}
  %   and the set of $\epsilon$-exploratory policies for $M$ is
  %   \begin{align*}
  %     \Pi_M^\epsilon := \{\pi \in \Pi_M : \forall (h, s) \in (\mathcal{S} \times \mathcal{A})^* \times \mathcal{S},\ \forall a \in \mathcal{A},\ \pi(\{a\} | h, s) \geq \epsilon \}.
  %   \end{align*}
  % \end{definition}

  % \begin{remark}

  % \end{remark}
\end{document}